\documentclass[letterpaper, 10 pt, conference]{ieeeconf} 
\IEEEoverridecommandlockouts            
\let\labelindent\relax
\usepackage{enumitem}

\usepackage{mathtools,amsthm}

\usepackage{graphics} 
\usepackage{epsfig}
\usepackage{times} 
\usepackage{amsmath,amssymb,amsfonts}
\usepackage{mathrsfs}
\usepackage{algorithmic}
\usepackage{graphicx}
\usepackage{textcomp}
\usepackage{varioref}
\usepackage{import}
\usepackage{framed,xcolor}
\usepackage{color}
\usepackage{comment}
\usepackage{multirow}
\usepackage{epstopdf}   
\usepackage{psfrag}
\usepackage{breqn}
\usepackage{float}
\usepackage{mathtools}
\usepackage{booktabs}
\usepackage{soul}
\usepackage{amsbsy}
\usepackage[bottom]{footmisc}
\usepackage{lineno}
\usepackage{cleveref}
\usepackage{microtype}
\usepackage{comment}
\newtheorem{theorem}{Theorem}

\newtheorem{Corollary}{Corollary}
\newtheorem{Lemma}{Lemma}
\newtheorem{Definition}{Definition}
\newtheorem{Remark}{Remark}

\newtheorem{Assumption}{Assumption}
\definecolor{arash}{rgb}{0.8,0.8,1}
\definecolor{seb}{rgb}{0.8,1,0.8}
\definecolor{seb2}{rgb}{0.5,.5,1}
\definecolor{arash2}{rgb}{0,.5,0}
\definecolor{wenqi}{rgb}{1,.75,0.79}
\definecolor{wenqi2}{rgb}{1,.75,0.79}

\DeclareMathOperator*{\argminC}{\arg\min}

\newcommand{\Argmin}[1]{\displaystyle \argminC_{#1}}
\newcommand{\Lim}[1]{\displaystyle \lim_{#1}}

\graphicspath{{Figures/}}
\newcommand{\vect}[1]{\ensuremath{\boldsymbol{\mathrm{#1}}}}
\usepackage{graphicx}
\usepackage{tabularx,booktabs}
\usepackage{color}
\usepackage{multicol}
\usepackage{amsmath,amssymb}
\usepackage{amsfonts}
\usepackage{textcomp}
\usepackage{psfrag}
\usepackage{multimedia}
\usepackage{fancybox}
\usepackage{comment}
\usepackage{soul}
\makeatletter
\newcommand{\biggg}{\bBigg@{1.6}}  

\makeatother

\definecolor{seb}{rgb}{0.8,1,0.8}
\definecolor{arash}{rgb}{0.8,0.8,1}

\newcounter{lastnote}

\usepackage{geometry}
\title{\LARGE \bf
Quasi-Newton Iteration in Deterministic Policy Gradient
}
\author{Arash Bahari Kordabad, Hossein Nejatbakhsh Esfahani, Wenqi Cai, Sébastien Gros
\thanks{The authors are with Department of Engineering Cybernetics, Norwegian University of Science and Technology (NTNU), Trondheim, Norway. E-mail:{\tt\small \{Arash.b.kordabad, hossein.n.esfahani, wenqi.cai, sebastien.gros\}@ntnu.no}}
 }
\begin{document}
\bstctlcite{IEEEexample:BSTcontrol}
\maketitle
\thispagestyle{empty}
\pagestyle{empty}

\begin{abstract} This paper presents a model-free approximation for the Hessian of the performance of deterministic policies to use in the context of Reinforcement Learning based on Quasi-Newton steps in the policy parameters. We show that the approximate Hessian converges to the exact Hessian at the optimal policy, and allows for a superlinear convergence in the learning, provided that the policy parametrization is rich. The natural policy gradient method can be interpreted as a particular case of the proposed method. We analytically verify the formulation in a simple linear case and compare the convergence of the proposed method with the natural policy gradient in a nonlinear example.
\end{abstract}
\section{INTRODUCTION}
Markov Decision Processes (MDPs) provide the standard framework for (stochastic) control problem. The Bellman equations provide the exact solution for a given MDP, and can be solved via Dynamic Programming (DP) \cite{bertsekas1995dynamic}. Unfortunately, this is impractical because of the \textit{curse of dimensionality of DP}. In practice, Reinforcement learning (RL) provides model-free tools to obtain an approximate solutions for the MDPs.
\par Deterministic policy gradient algorithms are widely used in RL with continuous action spaces \cite{bertsekas2019reinforcement}. These methods attempt to learn the optimal parameters of a parameterized policy $\vect \pi_{\vect \theta}$ using only state transitions observed on the real system. These methods commonly use gradient descent methods to optimize a discounted sum of stage costs, called closed-loop performance $J(\vect\theta)$. Depending on the policy type, these approaches are divided into the deterministic and the stochastic policy gradient methods. In the stochastic policy gradient methods, a parametrized distribution of action $\vect a$ conditioned on each state $\vect s$ taking the form of $\vect\pi_{\vect\theta}(\vect a|\vect s)$ is considered, while deterministic policy methods use $\vect a=\vect \pi_{\vect \theta}(\vect s)$ to specify a deterministic action for each state $\vect s$. Both methods adjust the parameter vector $\vect\theta$ in order to optimize $J$. In practice, stochastic policy gradient may need more data when the action space has many dimensions \cite{silver2014deterministic}. Hence, in this paper we focus on the deterministic policies. 
\par Unfortunately, the convergence rate of classical gradient descent is limited, especially when the Hessian of closed-loop performance $J$ is far from a scalar multiple of the Identity matrix~\cite{nocedal2006numerical}. In \cite{fazel2018global}, the global convergence of policy gradient methods has been investigated for the Linear Quadratic Regulator (LQR) problems. Various studies propose to use the Hessian of the policy performance in a Newton-type methods in order to deliver a faster learning \cite{furmston2016approximate}. 
\par Natural policy gradient methods has been attracted many attentions in RL community recently due to its capability for better convergence~\cite{hansel2021benchmarking}. The efficiency of the natural policy gradient in RL was showed in \cite{Amari1998}. The natural policy gradient methods use the \textit{Fisher information matrix} as an approximate Hessian~\cite{kakade2002natural}. In \cite{ding2020natural}, a natural policy gradient method is developed for Constrained MDPs. A Quasi-Newton method is developed in \cite{givchi2015quasi} for Temporal Difference (TD) learning in order to get faster convergence. Natural Actor-critic has been investigated in \cite{peters2005natural}. Although the Fisher information matrix, as an approximation for the Hessian, is positive definite, it does not asymptotically converge to the exact Hessian necessarily, when the policy converges to the optimal policy~\cite{hansel2021benchmarking}. As a result, the rate of convergence of the natural policy gradient method is linear, i.e., the same as the regular gradient descent \cite{furmston2016approximate}. Therefore, providing an approximation of the Hessian (without imposing heavy computation) that converges to the exact Hessian at the optimal policy can improve the convergence rate. 
\par In this paper, we first derive a formulation for exact Hessian of deterministic policy performance with respect to the parameters. Then we provide a model-free approximation for the Hessian of the performance function $J$. We show that the approximate Hessian converges to the exact Hessian at the optimal policy when the parameterized policy is rich. As a result, it gives a superlinear convergence using a Quasi-Newton optimization.
\section{Hessian of the Policy Performance}\label{sec:Hessian}
In the RL context, the problem is assumed to be an MDP with an initial state distribution $p_1 (\vect s_0)$ and transition probability density $p(\vect s^{+}|\vect s,\vect a)$ where $\vect s \in \mathcal{S}\subseteq\mathbb{R}^{n_s}$, $\vect a \in \mathcal{A}\subseteq\mathbb{R}^{n_a}$, and $\vect s^+$ are the current state, input, and subsequent state, respectively, and $\vect s_0$ is the initial state. Every transition imposes a real scalar stage cost $\ell(\vect s,\vect a)$.  A deterministic policy denoted by $\vect \pi : \mathcal{S} \rightarrow \mathcal{A}$ specifies how the input $\vect a$ is chosen for each state $\vect s$. We consider a parametrized policy $\vect\pi_{\vect\theta}$ with parameter vector $\vect\theta \in \mathbb{R}^{n_{\theta}}$ and seek an optimal policy by adjusting parameter $\vect\theta$. The value function $V^{\vect\pi_{\vect\theta}}$ and action-value function $Q^{\vect\pi_{\vect\theta}} (\vect s,\vect a)$ are defined as follows:
\begin{subequations}\label{eq:bellman}
\begin{align}
Q^{\vect\pi_{\vect\theta}} (\vect s,\vect a) &= \ell(\vect s,\vect a)+\gamma\mathbb{E}_{p(\cdot|\vect s,\vect a)}\left[V^{\vect\pi_{\vect\theta}} (\vect s^+)|\vect s,\vect a\right], \\
V^{\vect\pi_{\vect\theta}} (\vect s) &= Q^{\vect \pi_{\vect\theta}}(\vect s,\vect\pi_{\vect\theta}(\vect s)),
\end{align}
\end{subequations}
where $\gamma\in (0, 1]$ is a discount factor. The performance objective $J(\vect\theta)$ is given as follows:
\begin{align} \label{eq:J}
J(\vect\theta) =\mathbb{E}_{\vect s_0}\left[V^{\vect \pi_{\vect\theta}}(\vect s_0)\right]=\mathbb{E}_{\vect s}\left[\ell(\vect s,\vect \pi_{\vect \theta}(\vect s))\right].
\end{align}
Note that we simplified the expectation notation $\mathbb{E}_{\vect s_0\sim p_1(\vect s_0)}[\cdot]=\mathbb{E}_{\vect s_0}[\cdot]$ and $\mathbb{E}_{\vect s}[\cdot]$ is taken over the expected sum of the discounted state distribution of the Markov chain in closed-loop with policy $\vect\pi_{\vect\theta}$. The purpose is solving the following optimization problem:
\begin{align} \label{eq:minJ}
\vect \theta^\star \in \Argmin{\vect \theta} J(\vect \theta).
\end{align}

In the following we make an assumption in order to guarantee the existence of the policy gradient and we recall the deterministic policy gradient theorem.
\begin{Assumption} \label{Assup1}
$p(\vect s'|\vect s,\vect a)$, $\nabla_{\vect a} p(\vect s'|\vect s,\vect a)$, $\vect \pi_{\vect\theta}(\vect s)$, $\nabla_{\vect\theta} \vect\pi_{\vect\theta}(\vect s)$, $\ell(\vect s,\vect a)$, $\nabla_{\vect a} \ell(\vect s,\vect a)$, $p_1(\vect s)$ are continuous in all parameters and variables $\vect s$, $\vect a$, $\vect s'$, $\vect \theta$. Also there exist $b$ and $L$ such that: 
\begin{align}
    &\sup_{\vect s}  p_1(\vect s)<b, \qquad\qquad \sup_{\{\vect a,\vect s,\vect s'\}}p(\vect s'|\vect s,\vect a)<b,\nonumber \\ &\sup_{\{\vect a,\vect s\}}\|\nabla_{\vect a} \ell(\vect s,\vect a)\|<L,\,\,\, \sup_{\{\vect a,\vect s,\vect s'\}}\|\nabla_{\vect a} p(\vect s'|\vect s,\vect a)\|<L.
\end{align}
Moreover, there exists a policy $\vect\pi_{\vect\theta}$ such that $J(\vect\theta)$ is finite.
\end{Assumption}
Assumption \ref{Assup1} is a standard assumption which is made in \cite{silver2014deterministic} in order to derive policy gradients. All derivatives are also bounded for a smooth enough $p$, such as the Gaussian distribution. Moreover, one can select the initial state distribution from a  bounded probability function. The existence of a policy that makes  the performance $J(\vect\theta)$ finite can be interpreted as a controllability assumption in the control literature. Policy gradient methods usually solve \eqref{eq:minJ} using gradient descent method, i.e., at each iteration $k$, we update $\vect\theta$ as follows:
\begin{align}
    \vect\theta_{k+1}=\vect\theta_k-\alpha \nabla_{\vect\theta} J(\vect\theta)|_{\vect\theta=\vect\theta_k},
\end{align}
where $\alpha$ is a positive step-size.
\begin{theorem}(Deterministic Policy Gradient)
Suppose that the MDP satisfies Assumption \ref{Assup1}; then $\nabla_{\vect a} Q^{\vect \pi_{\vect \theta}}$ exists and the deterministic policy gradient reads as:
\begin{align}\label{eq:dj}
    \nabla_{\vect \theta} J(\vect \theta) &=\mathbb{E}_{\vect s} \Big[\nabla_{\vect \theta} \vect \pi_{\vect \theta}(\vect s)\nabla_{\vect a} Q^{\vect \pi_{\vect \theta}}(\vect s,\vect a)\big|_{\vect a=\vect \pi_{\vect \theta}(\vect s)}\Big].
\end{align}
\end{theorem}
\begin{proof}
See in \cite{silver2014deterministic}.
\end{proof}
The next standard assumption will be made to ensure the  existence of the Hessian of the policy with respect to the policy parameters $\vect\theta$ and the Hessian of action-value function with respect to the input $\vect a$.
\begin{Assumption} \label{Assup1.5}
$\nabla^2_{\vect a} p(\vect s'|\vect s,\vect a)$, $\nabla^2_{\vect \theta} \vect \pi_{\vect \theta}(\vect s)$, $\nabla^2_{\vect a} \ell(\vect s,\vect a)$, are continuous in all parameters and variables $\vect s$, $\vect a$, $\vect s'$, $\vect \theta$. Moreover, there exists $M$ such that: \\
\begin{align}
    \sup_{\vect a,\vect s,\vect s'}\|\nabla^2_{\vect a} p(\vect s'|\vect s,\vect a)\|<M,\,\,\,\,\sup_{\vect a,\vect s}\|\nabla^2_{\vect a} \ell(\vect s,\vect a)\|<M.
\end{align}
\end{Assumption}
Similar to the assumption \ref{Assup1}, assumption \ref{Assup1.5} is made to derive the Hessian of the performance. In practice, the assumption is satisfied for a smooth enough transition $p$, policy $\vect\pi$ and stage cost $\ell$.
In the following we provide the exact Hessian of the deterministic policy performance with respect to the policy parameters.
\begin{Definition}
In this paper, we use the operation $\otimes:\mathbb{R}^{n_1\times n_2\times n_3}\times\mathbb{R}^{n_3}\rightarrow \mathbb{R}^{n_1\times n_2}$ for the product of a tensor $T$ and a vector $\vect v$, such that:
\begin{align}
    T \otimes \vect v \triangleq \sum_{i=1}^{{n_3}} v_{{i}}T_{{(:,:,i)}},
\end{align}
where scalar $v_{{i}}$ is the $i^{\mathrm{th}}$ element of vector $\vect v$ and matrix $[T_{(:,:,i)}]_{n_1\times n_2}$ is the $i^{\mathrm{th}}$ frontal slice of tensor $T$~\cite{braman2010third}.
\end{Definition}
\begin{theorem}(Deterministic Policy Hessian)\label{th:J''}
Under Assumptions \ref{Assup1} and \ref{Assup1.5}, $\nabla^2 _{\vect a} Q^{\vect \pi_{\vect \theta}}$ and the deterministic policy Hessian exist. The latter is given by:
\begin{align}\label{eq:J''}
\nabla^2_{\vect \theta} J(\vect \theta)= H(\vect \theta)+\gamma \Lambda(\vect\theta),
\end{align}
where $H(\vect\theta)$ and $\Lambda(\vect\theta)$ are defined as follows:
\begin{subequations}
\begin{align}
      H(\vect \theta) \overset{\Delta}{=} & \mathbb{E}_{\vect s} \Big[ \nabla_{\vect \theta}^2 \vect \pi_{\vect \theta}(\vect s) \otimes \nabla_{\vect a} Q^{\vect \pi_{\vect \theta}}(\vect s,\vect a)\Big|_{\vect a=\vect \pi_{\vect \theta}}+\label{eq:H} \\  &\quad\, \nabla_{\vect \theta} \vect \pi_{\vect \theta}(\vect s)\nabla_{\vect a}^2 Q^{\vect \pi_{\vect \theta}}(\vect s,\vect a)\Big|_{\vect a=\vect \pi_{\vect \theta}}\nabla_{\vect \theta} \vect \pi_{\vect \theta}(\vect s)^\top\Big],\nonumber\\
      \Lambda(\vect\theta)\overset{\Delta}{=}&  \mathbb{E}_{\vect s}\bigg[\int  \nabla_{\vect \theta} p(\vect s'|\vect s,\vect \pi_{\vect \theta} (\vect s)) \nabla_{\vect \theta} V^{\vect \pi_{\vect \theta}}(\vect s')^\top \mathrm{d}\vect s'+ \label{eq:Lam} \\ &\quad \int \nabla_{\vect \theta} V^{\vect \pi_{\vect \theta}}(\vect s')  \nabla_{\vect \theta} p(\vect s'|\vect s,\vect \pi_{\vect \theta} (\vect s))^\top\mathrm{d}\vect s'\bigg].  \nonumber
\end{align}
\end{subequations}
\end{theorem}
\begin{proof}
See Appendix.
\end{proof}
\par The terms in \eqref{eq:H} only depend on the policy and the action-value function, but the terms in \eqref{eq:Lam} depend on the gradient of the transition probability $p(\vect s'|\vect s,\vect a)$, which is difficult to calculate directly from data. Hence, we use $H(\vect \theta)$ as a model-free approximator of the exact Hessian $\nabla^2_{\vect\theta}J$. Next section, we will 
show that the approximate Hessian $H(\vect \theta)$ converges to the exact Hessian $\nabla^2_{\vect\theta}J$ at the optimal policy.
\begin{Remark}
Note that one can approximate $p(\vect s'|\vect s,\vect a)$ from observed data in order to obtain a more accurate Hessian, e.g., using system identification techniques~\cite{martinsen2020combining}. Such an estimation can require a heavy computation if the state-action space of the problem is not small. Hence, in order to provide a model-free approximator and for sake of brevity we ignore such evaluation in this paper.
\end{Remark}
\section{Quasi-Newton Policy Improvement}\label{sec:quasi}
Quasi-Newton methods are alternative to Newton's approach where the Hessian of the cost function is unavailable or too expensive to compute at every iteration. A  Quasi-Newton update rule for the optimization problem \eqref{eq:minJ} can be written as follows:
\begin{align}\label{eq:QN}
\vect \theta_{k+1}=\vect \theta_k-\alpha H^{-1}(\vect \theta_k) \nabla_{\vect \theta} J(\vect \theta)|_{\vect \theta=\vect \theta_k},
\end{align}
where $H$ is an approximation of Hessian of the performance function $J$. Note that using a Hessian in the policy optimization is advantageous when the different parameters would require very different step sizes in a first-order method, i.e., when $\nabla^2J$ is far from being a multiple of the identity matrix. This is often the case in practice, unless a pre-scaling is performed on the policy formulation. From the computational viewpoint, the Hessian of a policy is usually dense, and it can be troublesome to use in \eqref{eq:QN} for a policy parametrization using a very large number of parameters. Hence the proposed second-order method is arguably best for policies using a few dozens, up to a few hundreds of parameters. E.g., policy parametrizations based on model predictive control techniques fall in that range of parameters~\cite{gros2019data}. Next mild assumptions are made to allow one to use the Newton-type optimization in the policy gradient methods. 
\begin{Assumption} \label{assum3} 
\begin{enumerate}
    \item The parameterized policy $\vect \pi_{\vect \theta}$ is rich enough. I.e., there exists $\vect\theta^\star$ such that $\vect \pi_{\vect \theta^\star}(\vect s)=\vect \pi^\star(\vect s)$.
\item     $J(\vect \theta)$ has a Lipschitz continuous Hessian and $\nabla^2_{\vect\theta}J(\vect \theta)^{-1}$ exists in a neighbourhood of $\vect \theta^\star$.
\end{enumerate}
\end{Assumption}
The first statement of Assumption \ref{assum3} is a standard assumption in the theoretical developments associated to the policy gradient method. For instance, for a Linear dynamic with Quadratic cost, a policy in the form of  $\vect\pi_{\vect\theta}(\vect s)=\Theta_1\vect s+\Theta_2$ with proper matrix dimension $\Theta_1$ and $\Theta_2$ satisfies Assumption \ref{assum3}.1, where $\vect\theta=\{\Theta_1,\Theta_2\}$. In practice, for a general problem such assumption is satisfied approximately by choosing a generic function approximator for the deterministic policy, e.g., Deep Neural Networks~\cite{franccois2018introduction} and Fuzzy Neural Networks~\cite{bahari2020emotional}. Then a richer policy satisfies the assumption asymptotically. A key consequence of this assumption is that the optimal policy $\vect\pi^\star$ is independent of the distribution of the initial state $p_1(\vect s_0)$. The second statement guarantees the continuity of the Hessian and allows one to use a Quasi-Newton approach.
\begin{Lemma}\label{lemma:f0}
Assume that $\vect f:\mathbb{R}^n \rightarrow \mathbb{R}^m$ is a bounded, continuous function of $\vect x\in\mathbb{R}^n$ and for any probability density $ g(\vect x)$, we have
 $   \mathbb{E}_{\vect x\sim  g}[\vect f(\vect x)]=\vect 0.$
Then
$
    \vect f(\vect x)= \vect 0 
$ holds almost everywhere in Lebesgue measure.
\end{Lemma}
\begin{proof} If $
    \vect f(\vect x)\neq \vect 0 
$ holds on a measurable set, then there exists a probability density $\tilde g$ on that set such that $
    \mathbb{E}_{\vect x\sim  \tilde g}[\vect f(\vect x)]\neq 0
$
\end{proof}
\begin{theorem}\label{theorem: J''=H}
Under Assumptions \ref{Assup1}-\ref{assum3}, the approximate Hessian $H(\vect\theta)$ converges to the exact Hessian $\nabla^2_{\vect\theta}J(\vect\pi_{\vect\theta})$ at the optimal policy, i.e.,
\begin{align}\label{eq:Lam0}
    \lim_{\vect\theta\rightarrow\vect\theta^\star}\Lambda(\vect\theta)=0.
\end{align}
\end{theorem} 
\begin{proof}
The initial distribution $p_1(\vect s_0)$ is independent of the policy parameters $\vect\theta$. From the optimality condition of \eqref{eq:J}, we have:
\begin{align}
    \nabla_{\vect\theta}J(\vect\theta)=\nabla_{\vect\theta}\mathbb{E}_{\vect s_0}[V^{\vect\pi_{\vect\theta}}(\vect s_0)]=\mathbb{E}_{\vect s_0}[\nabla_{\vect\theta}V^{\vect\pi_{\vect\theta}}(\vect s_0)]=0 \nonumber
\end{align}
at $\vect\theta=\vect\theta^\star$ for any initial distribution $p_1(\vect s_0)$ (Assumption \ref{assum3}.1). Using Lemma \ref{lemma:f0}, it implies  
$
    \nabla_{\vect\theta}V^{\vect\pi_{\vect\theta}}(\vect s)\equiv 0
$
at $\vect\theta=\vect\theta^\star$. Under Assumptions \ref{assum3} and for any bounded $\nabla_{\vect\theta} p$, it reads:
\begin{align}
&\int  \nabla_{\vect \theta} p(\vect s'|\vect s,\vect \pi_{\vect \theta} (\vect s)) \nabla_{\vect \theta} V^{\vect \pi_{\vect \theta}}(\vect s')^\top \mathrm{d}\vect s'=\nonumber\\  &\int \nabla_{\vect \theta} V^{\vect \pi_{\vect \theta}}(\vect s')  \nabla_{\vect \theta} p(\vect s'|\vect s,\vect \pi_{\vect \theta} (\vect s))^\top\mathrm{d}\vect s'=0   
\end{align}
at $\vect\theta=\vect\theta^\star$. Then, from the continuity of the Hessian (Assumption \ref{assum3}.2) and \eqref{eq:Lam}, it implies \eqref{eq:Lam0}. Note that Assumption \ref{Assup1} guarantees the boundedness of $\nabla_{\vect\theta} p$.
\end{proof}
Next theorem provides necessary and sufficient conditions for the superlinear\footnote{The sequence $x_k$ is said to converge superlinearly to $L$ if $\lim_{k\rightarrow\infty} \frac{|x_{k+1}-L|}{|x_{k}-L|}=0$.} convergence of the Quasi-Newton method.
\begin{theorem}\label{th:superlinear}(superlinear convergence of Quasi-Newton methods)
Suppose that $f: \mathbb{R}^n \rightarrow \mathbb{R}$ is twice continuously differetiable. Consider the iteration $x_{k+1}=x_k-B^{-1}_k \nabla f_k$. Let us assume that $\{ x_k \}$ converges to a point such that $\nabla f(x^\star)=0$ and $\nabla^2 f(x^\star)$ is positive definite. Then $\{ x_k \}$ converges superlinearly to $x^\star$ if and only if:
\begin{align}
    \Lim{k\rightarrow \infty} \frac{\| (B_k-\nabla^2 f(x^\star))B^{-1}_k \nabla f_k\|}{\|B^{-1}_k \nabla f_k\|} =0.
\end{align}
\end{theorem}
\begin{proof}
See Theorem 3.7 in \cite{nocedal2006numerical}.
\end{proof}
Next corollary concludes that the proposed Hessian implies a superlinear converges.
\begin{Corollary} \textbf{(From theorem \ref{theorem: J''=H} and \ref{th:superlinear}):}
Under assumption \ref{assum3} and the assumptions in the theorem \ref{th:superlinear}, the policy parameters $\vect \theta_k$ converge to the optimal policy parameters $\vect \theta^\star$ superlinearly, when $H(\vect \theta)$ defined in \eqref{eq:H} is an approximator of the exact Hessian \eqref{eq:J''} with $J(\vect\theta)$ defined in \eqref{eq:J} and the Quasi-Newton update rule \eqref{eq:QN} is used.  
\end{Corollary}
Natural policy gradient utilizes Fisher information matrix as its approximate Hessian in the policy gradient method. The Fisher matrix for deterministic policies can be written as follows~\cite{Covariant}:
\begin{align}\label{eq:fisher}
 F(\vect \theta)=\mathbb{E}_{\vect s}\left[\nabla_{\vect \theta} \vect \pi_{\vect \theta} (\vect s) \nabla_{\vect \theta} \vect \pi_{\vect \theta}(\vect s)^\top\right].   
\end{align}
The following corollary connects our proposed Hessian with the Fisher Information matrix.
\begin{Corollary}
Fisher Information matrix, defined in \eqref{eq:fisher}, is positive definite and by comparison with \eqref{eq:H} and this matrix can be written equal to \eqref{eq:H} under the following conditions:
\begin{enumerate}
    \item $\nabla^2_{\vect a} {Q}^{\vect \pi_{\vect \theta}} (\vect s,\vect a)|_{\vect a=\vect \pi_{\vect \theta}}=I$,
    \item $\nabla^2_{\vect \theta} \vect \pi_{\vect \theta} (\vect s) \otimes \nabla_{\vect a} {Q}^{\vect \pi_{\vect \theta}} (\vect s,\vect a)|_{\vect a=\vect \pi_{\vect \theta}}=0$.
\end{enumerate}
Then clearly $F(\vect\theta)$ does not converge to the exact Hessian at the optimal policy necessarily. I.e., the parameters will not converge superlinearly to the optimal parameters if the Fisher information matrix is used as a Hessian approximation (see Theorem \ref{th:superlinear}).
\end{Corollary}
\begin{Remark} Under assumptions \ref{Assup1}-\ref{assum3}, $H(\vect\theta)$ is positive definite in a neighborhood of $\vect\theta^\star$. Nevertheless $H(\vect\theta)$ is not necessarily positive definite for a parameter $\vect\theta$ that is far from the optimal parameter $\vect\theta^\star$ because of the term $\nabla_{\vect \theta}^2 \vect \pi_{\vect \theta}(\vect s) \otimes \nabla_{\vect a} Q^{\vect \pi_{\vect \theta}}(\vect s,\vect a)|_{\vect a=\vect \pi_{\vect \theta}}$, while the Fisher information matrix $F(\vect\theta)$ is (semi) positive definite by construction. A regularization of $H$ may be needed in practice, and one can use the Fisher information matrix $F$ to regularize the approximate Hessian $H$, when $H$ is not positive definite. This regularization can be applied using a Hessian in the form of $H+\beta F$ at every step, where $\beta\geq 0$ is a constant that must be ideally selected at every step.  However, other methods e.g.,
trust-region methods can effectively take advantage of
indefinite Hessian approximations.\end{Remark}
\begin{Remark}
Many RL methods deliver a sequence of parameters $\vect\theta_k$ that is stochastic by nature, because they are based on measurements taken from a stochastic system. From the theoretical viewpoint, all of the results in this paper are valid for large data sets, where sample averages converge to the true
expectations. However, in practice, one can use the method to improve the stochastic convergence rate and derive an extension of the current theorems. 
\end{Remark}
\section{Analytical Example}\label{sec:ana}
In this section, we consider a simple Linear Quadratic Regulator  (LQR) problem in order to verify the method analytically. Consider the following scalar linear dynamics:
\begin{align}\label{eq:ex1:dyn}
     s^+= s+ a+ w,
\end{align}
where $w\sim\mathcal{N}(0,\sigma^2)$, i.i.d., $\mathbb{E}_w[wa]=0$ and $\mathbb{E}_w[ws]=0$. Transition probability of the MDP \eqref{eq:ex1:dyn} reads as follows:
\begin{align}\label{eq:ex:p}
    p(s'|s,a)= \frac{1}{\sqrt{2\pi}\sigma}\exp\Big(-\frac{( s'-s-a)^2}{2\sigma^{2}} \Big).
\end{align}
Initial state distribution is $p_1(s_0)\sim\mathcal{N}(0,\sigma_0^2)$ , deterministic policy reads as $\pi_{\theta}=-\theta s$ and stage cost is $\ell(s,a)=0.5(s^2+a^2)$. We assume value function in the from of $V^{\pi_\theta}(s)=p_\theta s^2+q_\theta$ and we show it will satisfy the fundamental Bellman equations \eqref{eq:bellman}, then we have:
\begin{align}
    V^{\pi_\theta}(s)&=\ell(s,\pi_\theta(s))+\gamma \mathbb{E}_w[V^{\pi_\theta}(s-\theta s+w)]\\&=0.5s^2(1+\theta^2)+\gamma(1-\theta)^2p_\theta s^2+\gamma p_\theta \sigma^2+\gamma q_\theta.\nonumber
\end{align}
It implies:
\begin{align}
    p_\theta=\frac{0.5(1+\theta^2)}{1-\gamma(1-\theta)^2}, & \qquad
    q_\theta=\frac{\gamma\sigma^2}{1-\gamma} p_\theta.
\end{align}
Using the Bellman equations \eqref{eq:bellman}, the action-value function $Q^{\pi_\theta}(s,a)$ can be evaluated as follows:
\begin{align}
    Q^{\pi_\theta}(s,a)&=\ell(s,a)+\gamma\mathbb{E}\left[V^{\pi_\theta}(s^+|s,a)\right]\\
    &=0.5(s^2+a^2)+\gamma \mathbb{E} [p_\theta(s+a+w)^2+q_\theta]\nonumber\\&=(0.5+\gamma p_\theta)s^2+2\gamma p_\theta sa+(0.5+\gamma p_\theta)a^2+ q_\theta. \nonumber
\end{align}
One can check the identity $V^{\pi_\theta}(s)=Q^{\pi_\theta}(s,\pi(s))$. Then:
\begin{subequations}
\begin{align}
    \nabla_\theta \pi_\theta \nabla_a Q^{\pi_\theta}(s,a)|_{a=\pi_\theta}&=
    \frac{\gamma\theta^2+\theta-\gamma}{1-\gamma(1-\theta)^2}s^2\label{eq:ex:j'0} \\
      \nabla_\theta \pi_\theta \nabla^2_a Q^{\pi_\theta}(s,a)|_{a=\pi_\theta}\nabla_\theta\pi_\theta &=s^2(1+2\gamma p_\theta).\label{eq:ex:j''0}
\end{align}
\end{subequations}
Note that $\nabla^2_\theta \pi_\theta=0$. The closed-loop performance $J$ reads:
\begin{align}\label{eq:ex:j}
    J(\theta)=\mathbb{E}_{s_0}[V^{\pi_\theta}(s_0)]=\frac{0.5(1+\theta^2)}{1-\gamma(1-\theta)^2}(\sigma_0^2+\frac{\gamma\sigma^2}{1-\gamma}).
\end{align}
Then, by taking derivation of $J$ with respect to the parameters $\vect\theta$:
\begin{align}\label{eq:ex:j'1}
    J'(\theta)= \frac{\gamma\theta^2+\theta-\gamma}{(1-\gamma(1-\theta)^2)^2}(\sigma_0^2+\frac{\gamma\sigma^2}{1-\gamma}).
\end{align}
From policy gradient \eqref{eq:dj} and \eqref{eq:ex:j'0}, we can write:
\begin{align}\label{eq:ex:j'2}
    J'(\theta)&=\mathbb{E}_{s}[\nabla_\theta \pi_\theta \nabla_a Q^{\pi_\theta}(s,a)|_{a=\pi_\theta}]\nonumber\\&=\mathbb{E}_{s}[\frac{\gamma\theta^2+\theta-\gamma}{1-\gamma(1-\theta)^2}s^2].
\end{align}
Then \eqref{eq:ex:j'1} and \eqref{eq:ex:j'2} imply:
\begin{align}
    \mathbb{E}_{s}[s^2]=\frac{(\sigma_0^2+\frac{\gamma\sigma^2}{1-\gamma})}{1-\gamma(1-\theta)^2}.
\end{align}
From \eqref{eq:ex:j}, the exact Hessian of the performance $J$ reads:
\begin{align}\label{eq:exJ''}
    &J''(\theta)=p''_\theta(\sigma_0^2+\frac{\gamma\sigma^2}{1-\gamma})\\&
    =\frac{-2\gamma^2\theta^3-3\gamma\theta^2+6\gamma^2\theta-4\gamma^2+\gamma-1}{(1-\gamma(1-\theta)^2)^3}(\sigma_0^2+\frac{\gamma\sigma^2}{1-\gamma}).\nonumber
\end{align}
From \eqref{eq:H} and \eqref{eq:ex:j''0}, the approximate Hessian $H(\theta)$ reads:
\begin{align}\label{eq:ex:H}
    H(\theta)&=\mathbb{E}_{s}[s^2(1+2\gamma p_\theta)]\nonumber\\&=\frac{1+2\gamma\theta}{(1-\gamma(1-\theta)^2)^2}(\sigma_0^2+\frac{\gamma\sigma^2}{1-\gamma}).
\end{align}
From \eqref{eq:Lam} and \eqref{eq:ex:p}, we can write:
\begin{align}\label{eq:ex:Lam}
    \Lambda(\theta)=&2\int_{\mathcal{S}} \nabla_\theta V^{\pi_\theta}(s') \nabla_{\theta}p( s'| s, \pi_{\theta})\mathrm{d} s'\\=&2\int_{-\infty}^{\infty} -p'_\theta((s')^2+\frac{\gamma\sigma^2}{1-\gamma})\frac{ s( s'- s+ \theta s)}{\sqrt{2\pi}\sigma^{3}}\nonumber\\&\quad \exp\Big(-\frac{( s'-s+\theta s)^2}{2\sigma^{2}} \Big)\mathrm{d} s'=-4p'_\theta s^2(1-\theta)\nonumber\\=&\frac{-4(\gamma\theta^2+\theta-\gamma)(1-\theta)}{(1-\gamma(1-\theta)^2)^3}(\sigma_0^2+\frac{\gamma\sigma^2}{1-\gamma}). \nonumber
\end{align}
Therefore, one can easily verify \eqref{eq:J''} by substitution \eqref{eq:exJ''}, \eqref{eq:ex:H} and \eqref{eq:ex:Lam} in \eqref{eq:J''}. Note that we used the following integration in \eqref{eq:ex:Lam}:
\begin{align}
    \int_{-\infty}^{\infty} (x^2+a)(x-b)\exp(-c(x-b)^2)\mathrm{d}x=\frac{\sqrt{\pi}b}{c^{\frac{3}{2}}},
\end{align}
where $a$, $b$ and $c>0$ are constraints. Fig. \ref{f15} (right) compares the exact Hessian $\nabla^2_\theta J(\theta)$, the proposed approximate Hessian $H(\theta)$ and the Fisher matrix $F(\theta)$ for this example with $\gamma=0.9$ and $\sigma_0^2=\sigma^2=0.1$. As can be seen, $\nabla^2_\theta J$ meets $H(\theta)$ at the optimal parameter. Fig. \ref{f15} (left) shows the superlinear convergence of the policy parameters during the learning using Quasi-Newton policy gradient method, while the (first order) policy gradient method and natural policy gradient method result a linear convergence during the learning.
\begin{figure}[ht!]
	\centering
\includegraphics[width=0.48\textwidth]{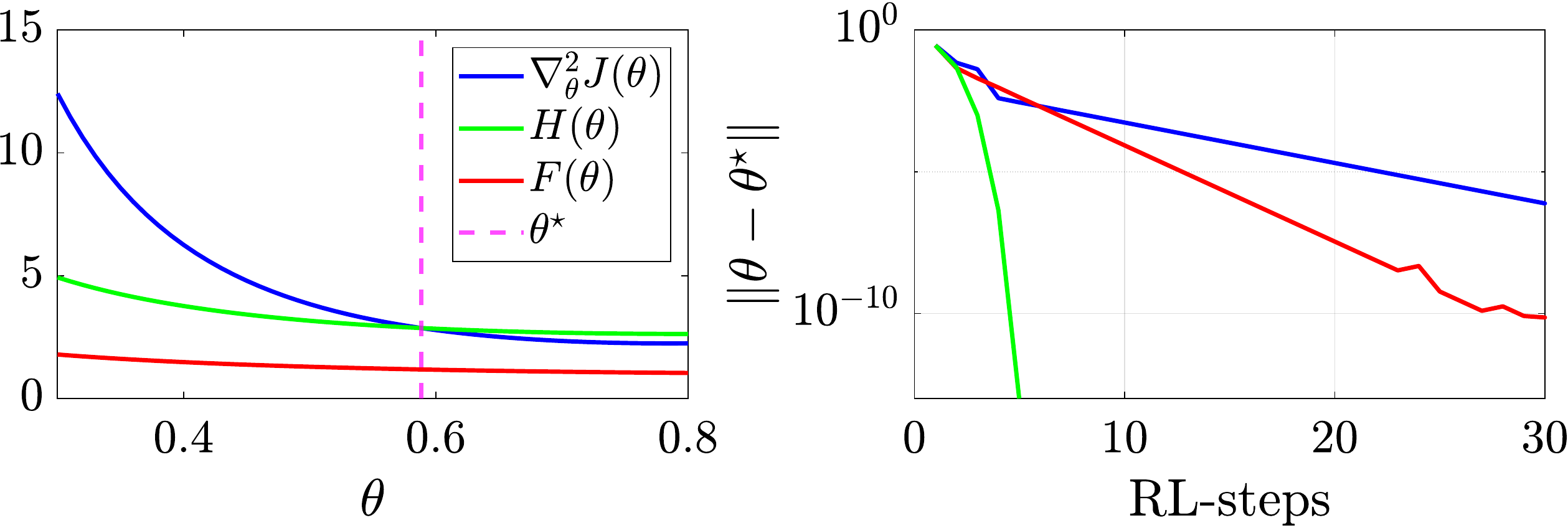}
	\caption{Right: Superlinear convergence of the proposed method. blue: policy gradient method, red: natural policy gradient method, green: proposed method. Left: Comparison of the exact Hessian $\nabla^2_\theta J(\theta)$, the proposed approximate Hessian $H(\theta)$ and the Fisher matrix $F(\theta)$.}
	\label{f15}
\end{figure}
\section{Numerical Simulation}\label{sec:sim}
Cart-Pendulum balancing is a well-known benchmark in the RL community. The dynamics of a cart-pendulum system, shown in fig. \ref{figdyn}, reads as:
\begin{subequations}\label{eq:cp}
\begin{align}
  (M+m)\ddot{x}+\frac{1}{2}ml\ddot{\phi}\cos\phi &=\frac{1}{2}ml\dot{\phi}^2\sin\phi+u,\\
  \frac{1}{3}ml^2\ddot{\phi}+\frac{1}{2}ml\ddot{x}\cos\phi &=-\frac{1}{2}mgl\sin\phi,
\end{align}
\end{subequations}
where $M$ and $m$ are the cart mass and pendulum mass, respectively, $l$ is the pendulum length and $\phi$ is its angle from the vertical axis. Force $u$ is the control input, $x$ is the cart displacement and $g$ is gravity. We used the Runge-Kutta $4^{\mathrm{th}}$-order method to discretize \eqref{eq:cp} with a sampling time $\mathrm{d}t=0.1\mathrm{s}$ and cast it in the form of $\vect s^+=\vect f(\vect s,\vect a)+\vect\xi$, where $\vect s=[\dot x,x,\dot \phi,\phi]^\top$ is the state, $\vect a=u$ is the input, $\vect\xi$ is a Gaussian noise and $\vect f$ is a nonlinear function representing \eqref{eq:cp} in discrete time. A stabilizing quadratic stage cost is considered as $\ell(\vect s,\vect a)=\vect s^\top\vect s+0.01\vect a^\top \vect a$, and the deterministic policy is considered in the form of $\vect\pi_{\vect\theta}=-\vect\theta\vect s$.
\begin{figure}[ht!]
	\centering
	{\def\svgwidth{0.39\textwidth}
			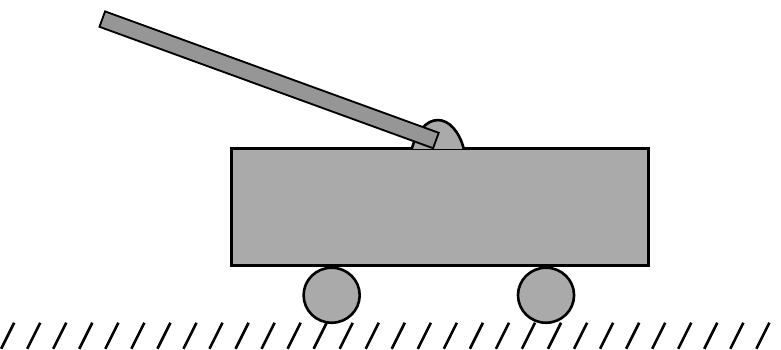
			}
	\caption{The cart-pendulum system. We use $M=0.5\mathrm{kg}$, $m=0.2\mathrm{kg}$, $l=0.3\mathrm{m}$ and $g=9.8\mathrm{m/s^2}$ for the simulation.}
	\label{figdyn}
\end{figure}
Fig. \ref{f2} (right) shows the closed-loop performance $J$ using the proposed Hessian $H(\vect\theta)$ (green) and natural policy gradient method (red). Moreover, the deterministic policy parameters $\vect\theta$ is shown in fig. \ref{f2} (left). 
\begin{figure}[ht!]
	\centering
\includegraphics[width=0.48\textwidth]{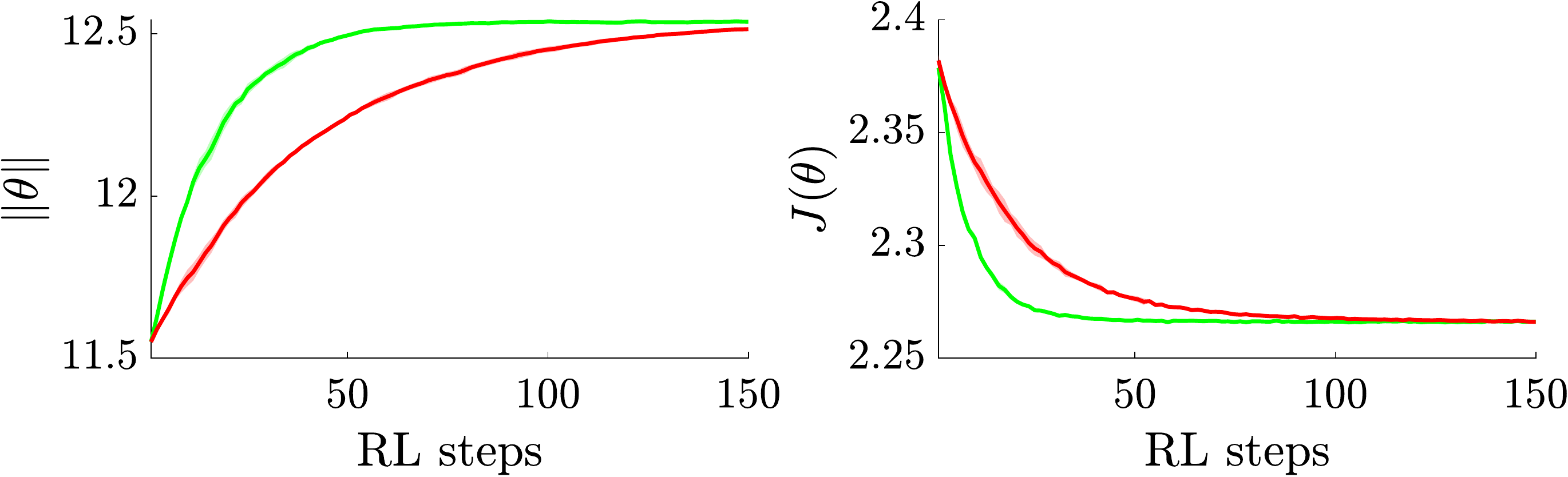}
	\caption{Right: Closed-loop performance $J(\vect\theta)$; Left: Convergence of the policy parameters $\vect\theta$ using the proposed Hessian (green) and natural policy gradient method (red).}
	\label{f2}
\end{figure}
\section{Conclusion}\label{sec:conc}
In this work, we provided a Hessian approximation for the performance of deterministic policies. We use the model-independent terms of the exact Hessian as an approximate Hessian, and we showed that the resulting approximate Hessian converges to the exact Hessian at the optimal policy. Therefore, the approximate Hessian can be used in the Quasi-Newton optimization to provide a superlinear convergence. We analytically verified our formulation in a simple example, and we compare our method with the natural policy gradient in a cart-pendulum system. In the future, we will investigate actor-critic algorithms for the proposed Hessian.
\typeout{}
 \bibliographystyle{IEEEtran}
\bibliography{ACC2022}
\appendix
    \section*{Proof of Theorem \ref{th:J''} }\label{appendix:D.P.H}
\setcounter{equation}{0}
\renewcommand{\theequation}{A.\arabic{equation}}
\begin{proof} 
We first calculate the Hessian of $V^{\vect\pi_{\vect\theta}}(\vect s)$ as follows:
\begin{align} \label{eq:V''}
& \nabla^2_{\vect\theta} V^{\vect\pi_{\vect\theta}}(\vect s) = \nabla^2_{\vect\theta} Q^{\vect\pi_{\vect\theta}}(\vect s,\vect a)|_{\vect a=\vect \pi_{\vect \theta}(\vect s)} = \nonumber\\& \nabla^2_{\vect \theta} \left( \ell(\vect s,\vect \pi_{\vect \theta}(\vect s))+ \int_{\mathcal{S}} \gamma p(\vect s'|\vect s,\vect \pi_{\vect \theta}(\vect s)) V^{\vect \pi_{\vect \theta}}(\vect s')\mathrm{d}\vect s'\right)= \nonumber\\
&\qquad \nabla^2_{\vect \theta} \vect \pi_{\vect \theta} (\vect s) \otimes \nabla_{\vect a} \ell(\vect s,\vect a)|_{\vect a=\vect \pi_{\vect \theta}(\vect s)}+\nonumber\\&\qquad \nabla_{\vect \theta} \vect \pi_{\vect \theta}(\vect s) \nabla^2_{\vect a} \ell(\vect s,\vect a)|_{\vect a=\vect \pi_{\vect \theta}}\nabla_{\vect \theta} \vect \pi_{\vect \theta}(\vect s)^\top+ \nonumber\\&\qquad \nabla^2_{\vect \theta} \int_{\mathcal{S}} \gamma p(\vect s'|\vect s,\vect a) V^{\vect \pi_{\vect \theta}}(\vect s')\mathrm{d}\vect s'
\end{align}
The third term can be calculated as follows:
\begin{align}
 \nabla^2_{\vect \theta} & \int_{\mathcal{S}} \gamma p(\vect s'|\vect s,\vect a) V^{\vect \pi_{\vect \theta}}(\vect s')\mathrm{d}\vect s'= \nonumber\\&   \int_{\mathcal{S}} \gamma V^{\vect \pi_{\vect \theta}}(\vect s') \nabla^2_{\vect \theta} p(\vect s'|\vect s,\vect \pi_{\vect \theta}(\vect s)) \mathrm{d}\vect s'+\nonumber\\ & \int_{\mathcal{S}} \gamma \nabla_{\vect \theta} p(\vect s'|\vect s,\vect \pi_{\vect \theta}(\vect s)) \nabla_{\vect \theta} V^{\vect \pi_{\vect \theta}}(\vect s')^\top \mathrm{d}\vect s'+\nonumber\\& \int_{\mathcal{S}} \gamma \nabla_{\vect \theta} V^{\vect \pi_{\vect \theta}}(\vect s') \nabla_{\vect \theta} p(\vect s'|\vect s,\vect \pi_{\vect \theta}(\vect s))^\top \mathrm{d}\vect s'+ \nonumber \\ &
\int_{\mathcal{S}}\gamma p(\vect s'|\vect s,\vect \pi_{\vect \theta}(\vect s)) \nabla^2_{\vect \theta} V^{\vect \pi_{\vect \theta}}(\vect s')\mathrm{d}\vect s'
\end{align}
The first term can be extended as follows:
\begin{align}
 & \int_{\mathcal{S}} \gamma V^{\vect \pi_{\vect \theta}}(\vect s') \nabla^2_{\vect \theta} p(\vect s'|\vect s,\vect \pi_{\vect \theta}(\vect s)) \mathrm{d}\vect s'=  \\ & \int_{\mathcal{S}} \gamma  V^{\vect \pi_{\vect \theta}} (\vect s') \nabla^2_{\vect \theta} \vect \pi_{\vect \theta} (\vect s) \otimes \nabla_{\vect a} p(\vect s'|\vect s,\vect a)|_{\vect a=\vect \pi_{\vect \theta}(\vect s)}\mathrm{d}\vect s'+ \nonumber\\ &
 \int_{\mathcal{S}} \gamma V^{\vect \pi_{\vect \theta}} (\vect s') \nabla_{\vect \theta} \vect \pi_{\vect \theta} (\vect s) \nabla^2_{\vect a} p(\vect s'|\vect s,\vect a)|_{\vect a=\vect \pi_{\vect \theta}(\vect s)} \nabla_{\vect \theta} \vect \pi_{\vect \theta}(\vect s) ^\top \mathrm{d}\vect s'\nonumber
\end{align}
By rearranging \eqref{eq:V''}, we can write:
\begin{align} \label{eq:V'':1}
& \nabla^2_{\vect \theta} V^{\vect \pi_{\vect \theta}}(\vect s)= \nabla^2_{\vect \theta} \vect \pi_{\vect \theta} (\vect s) \otimes \nabla_{\vect a} ( \ell(\vect s,\vect a)|_{\vect a=\vect \pi_{\vect \theta}(\vect s)}+  \nonumber\\ & \int_{\mathcal{S}} \gamma p(\vect s'|\vect s,\vect a)|_{\vect a=\vect \pi_{\vect \theta}(\vect s)} V^{\vect \pi_{\vect \theta}}(\vect s')\mathrm{d}\vect s') +\nonumber \\ &  \nabla_{\vect \theta} \vect \pi_{\vect \theta} (\vect s) \nabla^2_{\vect a} ( \ell(\vect s,\vect a)|_{\vect a=\vect \pi_{\vect \theta}(\vect s)}+  \nonumber\\ & \int_{\mathcal{S}} \gamma p(\vect s'|\vect s,\vect a)|_{\vect a=\vect \pi_{\vect \theta}(\vect s)} V^{\vect \pi_{\vect \theta}}(\vect s')\mathrm{d}\vect s') \nabla_{\vect \theta} \vect \pi_{\vect \theta} (\vect s) ^\top+ \nonumber \\ &
\int_{\mathcal{S}} \gamma \nabla_{\vect \theta} p(\vect s'|\vect s,\vect \pi_{\vect \theta}(\vect s)) \nabla_{\vect \theta} V^{\vect \pi_{\vect \theta}}(\vect s')^\top \mathrm{d}\vect s'+ \nonumber\\ & \int_{\mathcal{S}} \gamma \nabla_{\vect \theta} V^{\vect \pi_{\vect \theta}}(\vect s') \nabla_{\vect \theta} p(\vect s'|\vect s,\vect \pi_{\vect \theta}(\vect s))^\top \mathrm{d}\vect s'+ \nonumber \\ &
 \int_{\mathcal{S}} \gamma p(\vect s'|\vect s,\vect \pi_{\vect \theta}(\vect s)) \nabla^2_{\vect \theta} V^{\vect \pi_{\vect \theta}}(\vect s')\mathrm{d}\vect s'=  \nonumber\\ &
 \mathcal{F}_{\vect \theta}(\vect s)+\int_{\mathcal{S}} \gamma p(\vect s'|\vect s,\vect \pi_{\vect \theta}(\vect s)) \nabla^2_{\vect \theta} V^{\vect \pi_{\vect \theta}}(\vect s')\mathrm{d}\vect s'
\end{align}
where $\mathcal{F}_{\vect \theta}(\vect s)$ is defined as follows:
\begin{align}
     \mathcal{F}_{\vect \theta}(\vect s) \triangleq & \nabla^2_{\vect \theta} \vect \pi_{\vect \theta} (\vect s) \otimes \nabla_{\vect a} Q^{\vect \pi_{\vect \theta}}(\vect s,\vect a)|_{\vect a=\vect \pi_{\vect \theta}(\vect s)} +\nonumber \\ &  \nabla_{\vect \theta} \vect \pi_{\vect \theta} (\vect s) \nabla^2_{\vect a} Q^{\vect \pi_{\vect \theta}}(\vect s,\vect a)|_{\vect a=\vect \pi_{\vect \theta}(\vect s)} \nabla_{\vect \theta} \vect \pi_{\vect \theta} (\vect s) ^\top+ \nonumber \\ &
\int_{\mathcal{S}} \gamma \nabla_{\vect \theta} p(\vect s'|\vect s,\vect \pi_{\vect \theta}(\vect s)) \nabla_{\vect \theta} V^{\vect \pi_{\vect \theta}}(\vect s')^\top \mathrm{d}\vect s'+ \nonumber \\ & \int_{\mathcal{S}} \gamma \nabla_{\vect \theta} V^{\vect \pi_{\vect \theta}}(\vect s') \nabla_{\vect \theta} p(\vect s'|\vect s,\vect \pi_{\vect \theta}(\vect s))^\top \mathrm{d}\vect s'
\end{align}
where we used:
\begin{align}
  Q^{\vect \pi_{\vect \theta}}(\vect s,\vect a)= &  \ell(\vect s,\vect a)|_{\vect a=\vect \pi_{\vect \theta}(\vect s)}+ \nonumber\\ & \int_{\mathcal{S}} \gamma p(\vect s'|\vect s,\vect a)|_{\vect a=\vect \pi_{\vect \theta}(\vect s)} V^{\vect \pi_{\vect \theta}}(\vect s')\mathrm{d}\vect s'  
\end{align}
Now, we can go one step further for the last term of \eqref{eq:V'':1}:
\begin{align} \label{eq:V'':2}
    &\nabla^2_{\vect \theta} V^{\vect \pi_{\vect \theta}}(\vect s)= \mathcal{F}_{\vect \theta}(\vect s)+\int_{\mathcal{S}} \gamma p(\vect s'|\vect s,\vect \pi_{\vect \theta}(\vect s)) \mathcal{F}_{\vect \theta}(\vect s') \mathrm{d}\vect s'  +\nonumber\\ &\int_{\mathcal{S}} \int_{\mathcal{S}} \gamma^2 p(\vect s'|\vect s,\vect \pi_{\vect \theta}(\vect s))p(\vect s''|\vect s',\vect \pi_{\vect \theta}(\vect s')) \nabla^2_{\vect \theta} V^{\vect \pi_{\vect \theta}}(\vect s'')\mathrm{d}\vect s'\mathrm{d}\vect s''
\end{align}
where we have used the following equality:
\begin{align} 
\nabla^2_{\vect \theta} V^{\vect \pi_{\vect \theta}}(\vect s')=&\mathcal{F}_{\vect \theta}(\vect s')+\\&\int_{\mathcal{S}} \gamma p(\vect s''|\vect s',\vect \pi_{\vect \theta}(\vect s')) \nabla^2_{\vect \theta} V^{\vect \pi_{\vect \theta}}(\vect s'')\mathrm{d}\vect s''\nonumber
\end{align}
We can define:
\begin{align}
  p(\vect s\rightarrow \vect s'',2,\vect \pi_{\vect \theta})=\int_{\mathcal{S}} p(\vect s'|\vect s,\vect \pi_{\vect \theta}(\vect s))p(\vect s''|\vect s',\vect \pi_{\vect \theta}(\vect s'))\mathrm{d}\vect s'  \nonumber
\end{align}
and interpret it probability of transition from $\vect s$ to $\vect s''$ in $2$ steps by policy $\vect \pi_{\vect \theta}$. Then in last term we can alter integral notation $\vect s'' \rightarrow \vect s'$ and rewrite \eqref{eq:V'':2} as follows:
\begin{align}
       \nabla^2_{\vect \theta} V^{\vect \pi_{\vect \theta}}&(\vect s)= \mathcal{F}_{\vect \theta}(\vect s)+\int_{\mathcal{S}} \gamma p(\vect s'|\vect s,\vect \pi_{\vect \theta}(\vect s)) \mathcal{F}_{\vect \theta}(\vect s') \mathrm{d}\vect s'+ \nonumber \\ & \int_{\mathcal{S}} \gamma^2  p(\vect s\rightarrow \vect s',2,\vect \pi_{\vect \theta}) \nabla^2_{\vect \theta} V^{\vect \pi_{\vect \theta}}(\vect s')\mathrm{d}\vect s'
\end{align}
By continuing this procedure, we have:
\begin{align}
 \nabla^2_{\vect\theta} V^{\vect\pi_{\vect\theta}}(\vect s)&=
\int_{\mathcal{S}} \sum^\infty_{t=0} \gamma^t p(\vect s\rightarrow \vect s',t,\vect \pi_{\vect\theta}) \mathcal{F}_{\vect\theta}(\vect s') \mathrm{d}\vect s'
\end{align}
where
\begin{align}
  p(\vect s\rightarrow \vect s',t,\vect \pi_{\vect \theta})=\int_{\mathcal{S}} p(\vect s\rightarrow \hat{\vect s},t-1,\vect \pi_{\vect \theta})p(\vect s'|\hat{\vect s},\vect \pi_{\vect \theta}(\hat{\vect s}))\mathrm{d}\hat{\vect s}  \nonumber
\end{align}
starting from $p(\vect s\rightarrow \vect s',1,\vect \pi_{\vect \theta})=p(\vect s'|\vect s,\vect \pi_{\vect \theta}(\vect s))$.
Then, tacking the expectation over $p_1$ for Hessian of policy we have:
\begin{align}
\nabla^2_{\vect \theta} J(\vect \theta)&= \nabla^2_{\vect \theta} \int_{\mathcal{S}} p_1(\vect s)V^{\vect \pi_{\vect \theta}} (\vect s)\mathrm{d}\vect s= \\ & \int_{\mathcal{S}} p_1(\vect s) \nabla^2_{\vect \theta} V^{\vect \pi_{\vect \theta}} (\vect s)\mathrm{d}\vect s= \nonumber \\ & \int_{\mathcal{S}} \int_{\mathcal{S}} \sum^{\infty}_{t=0} \gamma^t p_1(\vect s) p(\vect s \rightarrow \vect s',t,\vect \pi_{\vect \theta})  \nonumber\\ &\Big[ \nabla_{\vect \theta}^2 \vect \pi_{\vect \theta}(\vect s')\otimes\nabla_{\vect a} Q_{\vect \pi_{\vect \theta}}(\vect s',\vect a)|_{\vect a=\vect \pi_{\vect \theta}(\vect s')}+ \nonumber \\ & \nabla_{\vect \theta} \vect \pi_{\vect \theta}(\vect s')\nabla_{\vect a}^2 Q^{\vect \pi_{\vect \theta}}(\vect s',\vect a)|_{\vect a=\vect \pi_{\vect \theta}(\vect s')}\nabla_{\vect \theta} \vect \pi_{\vect \theta}(\vect s')^\top + \nonumber\\ &
  \int_{\mathcal{S}} \gamma \nabla_{\vect \theta} p(\vect s''|\vect s',\vect \pi_{\vect \theta}(\vect s')) \nabla_{\vect \theta} V^{\vect \pi_{\vect \theta}}(\vect s'')^\top \mathrm{d}\vect s'' + \nonumber \\ & \int_{\mathcal{S}} \gamma \nabla_{\vect \theta} V^{\vect \pi_{\vect \theta}}(\vect s'') \nabla_{\vect \theta} p(\vect s''|\vect s',\vect \pi_{\vect \theta}(\vect s'))^\top \mathrm{d}\vect s'' \Big]\mathrm{d}\vect s'\mathrm{d}\vect s \nonumber
  \end{align}
Or equivalently:
\begin{align}
&\nabla^2_{\vect \theta} J(\vect \theta)= \mathbb{E}_{\vect s} \Big[ \nabla_{\vect \theta}^2 \vect \pi_{\vect \theta}(\vect s) \otimes \nabla_{\vect a} Q^{\vect \pi_{\vect \theta}}(\vect s,\vect a)|_{\vect a=\vect \pi_{\vect \theta}}+ \nonumber \\ & \nabla_{\vect \theta} \vect \pi_{\vect \theta}(\vect s)\nabla_{\vect a}^2 Q^{\vect \pi_{\vect \theta}}(\vect s,\vect a)|_{\vect a=\vect \pi_{\vect \theta}}\nabla_{\vect \theta} \vect \pi_{\vect \theta}(\vect s)^\top + \nonumber\\ &
  \int \gamma \nabla_{\vect \theta} V^{\vect \pi_{\vect \theta}}(\vect s') \nabla_{\vect \theta} p(\vect s'|\vect s,\vect \pi_{\vect \theta} (\vect s))^\top \mathrm{d}\vect s'+ \nonumber \\ & \int \gamma \nabla_{\vect \theta} p(\vect s'|\vect s,\vect \pi_{\vect \theta} (\vect s)) \nabla_{\vect \theta} V^{\vect \pi_{\vect \theta}}(\vect s')^\top \mathrm{d}\vect s' \Big]
\end{align}
where $\mathbb{E}_{\vect s}[\cdot]$ is taken over discounted state distribution of the Markov chain in closed-loop with policy $\vect\pi_{\vect\theta}$.
\end{proof}
\end{document}